\documentclass{article}

\usepackage{arxiv_default}

\usepackage[utf8]{inputenc}
\usepackage[T1]{fontenc}
\usepackage{hyperref}
\usepackage{url}
\usepackage{booktabs}
\usepackage{amsfonts}
\usepackage{amsmath}
\usepackage{amssymb}
\usepackage{amsthm}
\usepackage{nicefrac}
\usepackage{microtype}
\usepackage{xcolor}
\usepackage{graphicx}
\usepackage{subcaption}
\usepackage{algorithm}
\usepackage{algorithmic}
\usepackage{multirow}

\graphicspath{{figures/}{figures/unmask_order/}}

\newtheorem{theorem}{Theorem}[section]

\title{Decoding in Order-Agnostic Language Models: Chain-Rule Deviation and Uniform Spreading}

\author{
  Lin Yao$^{1,2}$ \\
  $^1$School of Computer Science, Shanghai Jiao Tong University, Shanghai, 200240, China \\
  $^2$Zhongguancun Academy, Beijing, 100097, China \\
  \texttt{lin.yao@sjtu.edu.cn}
}

\begin{document}

\maketitle

\begin{abstract}
Order-agnostic language models (OALMs), including discrete diffusion language models (dLLMs), are trained to predict masked tokens under arbitrary conditioning sets, allowing sequences to be generated or scored under arbitrary reveal orders at inference time. In LLaDA-2.1, we report three findings. First, the learned conditionals are not exact factorizations of a coherent joint distribution: changing only the reveal order shifts target log-likelihood by up to $0.49$ nats/token, so likelihood alone mixes content difficulty with path-dependent artifacts. Second, although confidence-first (CF) decoding is order-agnostic, its reveal orders are close to left-to-right (L2R) on content tokens. Third, we propose a complementary diagnostic based on the shape of the confidence trace. A uniform-spreading theorem shows that, at fixed total likelihood, target recoverability is maximized when per-step confidence is spread uniformly; the resulting deviation motivates $\mathrm{Var}(\log q_t)$ as a diagnostic for comparing decoding paths. Across C4 and four downstream benchmarks, low variance separates structured paths from random ordering, and variance is consistently associated with downstream correctness. These results support reporting mean confidence and confidence variance jointly when comparing OALM decoding paths.
\end{abstract}

\section{Introduction}
\label{sec:intro}

\begin{figure}[htbp]
    \centering
    \includegraphics[width=\textwidth]{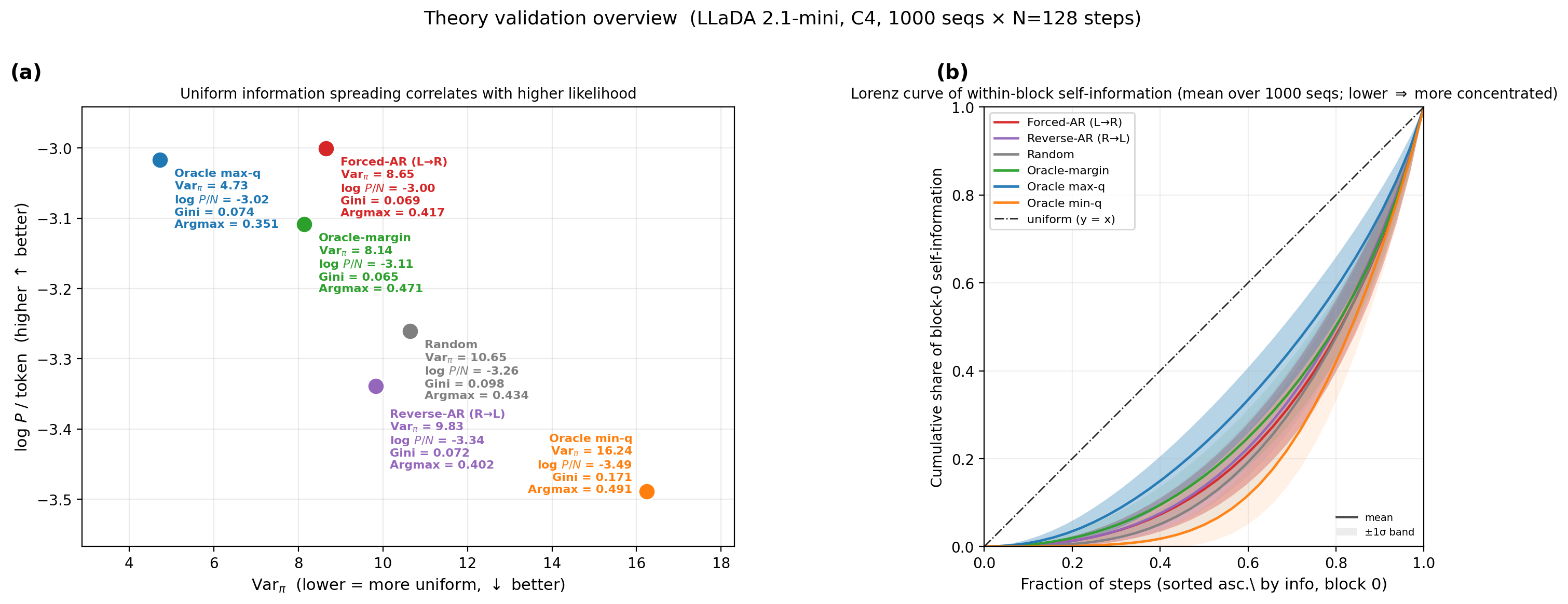}
    \caption{Fixed-sequence validation on $1{,}000$ C4 sequences (continuation $n{=}128$, block size $32$, LLaDA-2.1-mini).
    \textbf{(a)}~Per-strategy mean $\mathrm{Var}_\pi$ vs.\ mean $\log P/n$, annotated with the Gini-style concentration coefficient and single-step argmax accuracy.
    \textbf{(b)}~Lorenz curve of within-block self-information on the first block, averaged over $1{,}000$ sequences: $x$-axis is the fraction of steps sorted by ascending $-\log q_t$, $y$-axis is the cumulative share of $-\sum_t \log q_t$. The diagonal $y{=}x$ is the uniform schedule; curves below it indicate self-information concentrated in a few bottleneck steps.}
    \label{fig:theory_valid}
    \vspace{-0.75em}
\end{figure}

We use order-agnostic language models (OALMs) to denote a broad class of
language models trained to predict tokens under flexible conditioning patterns,
thereby supporting arbitrary inference orders. This class spans NADE-style order-agnostic
models~\cite{uria2014nade}, permutation language
modeling~\cite{yang2019xlnet}, arbitrary-order autoregressive
models~\cite{hoogeboom2021ardm}, and recent discrete diffusion language
models. In the dLLM branch, models such as MDLM~\cite{sahoo2024mdlm}
and LLaDA~\cite{nie2025llada} are trained with random masks at varying mask
ratios, so a single trained network~$p_\theta$
can infer in any reveal order. This makes dLLMs a natural testbed for isolating
reveal-order effects: we can hold the model and prompt fixed while varying only
the order, instantiated by representative strategies such as left-to-right
(L2R), right-to-left, random, and confidence-first (CF) variants including
\textsc{Max-prob} and \textsc{Top-margin}~\cite{ghazvininejad2019mask,
chang2022maskgit}.
The most direct way to compare orders is to score the same target
sequence~$\mathbf{x}^*$ along each path. At step~$t$, let $R_{t-1}$ be the tokens
already revealed by order~$\pi$, and define
$q_t(\pi)=p_\theta(x^*_{\pi(t)}\mid R_{t-1},c)$. One might then compare
orders by the summed score $\sum_t \log q_t(\pi)$.
If the model's conditionals were consistent with a coherent joint
distribution, the chain rule would make this sum independent of~$\pi$.
In our fixed-sequence evaluation on
LLaDA-2.1-mini, it is not: on $1{,}000$ randomly selected C4~\cite{raffel2020c4} examples, each split into a
prefix and a target continuation (Fig.~\ref{fig:theory_valid}), the choice
of reveal order shifts $\log P/n$ by up to $0.49$ nats per token. At
this magnitude, the conditional predictions should not be interpreted as
different exact factorizations of one joint distribution. Total
log-probability is therefore not order-invariant at the scale needed for
strategy comparison: it mixes content confidence with ordering artifacts
and is insufficient as a standalone quality signal.

To address this, we identify a complementary diagnostic axis. Even when the chain rule pins
$\prod_t q_t = P$ regardless of the reveal order~$\pi$, $P$ is not the
quantity greedy decoding achieves in practice. Instead, recovering the target
sequence becomes an allocation problem: at each step, the decoder commits
to an argmax over the vocabulary, and success depends not only on the
isolated confidence $q_t = p_\theta(x^*_{\pi(t)} \mid R_{t-1})$, but on
how decisively the target wins against its local competitors. Although the reveal order does not change the model parameters, it determines which
contexts $R_{t-1}$ are encountered, and thus how the total
budget $P$ is allocated across tokens to maintain a winning margin.
This sensitivity to the allocation shape arises because greedy
decoding acts as a chain of local decisions: a single bottleneck
step---where the allocated $q_t$ is insufficient to secure the argmax---can
determine whether the path succeeds, regardless of how high the
probabilities are at other steps. Consequently, a path with $n$
moderately confident steps can be more reliable than one that is highly
confident overall but contains a single low-confidence bottleneck. We
formalize this through the uniform-spreading theorem
(Theorem~\ref{thm:uniform}), proving that at a fixed total $P$, target recoverability is
uniquely maximized when confidence is spread uniformly ($q_t = P^{1/n}$),
with the structural cost of non-uniformity captured by
$\mathrm{Var}(\log q_t)$. Thus, $\log P$ and $\mathrm{Var}(\log q_t)$
define two complementary diagnostic axes: the former measures the
model's total budget, while the latter captures the efficiency of the
chosen reveal order.

Empirically, the variance axis explains both strategy-level and downstream
patterns. On C4 generation, structured one-token-per-step strategies
(\textsc{Forced-AR}, \textsc{Max-prob}, and \textsc{Top-margin}) have similar
external perplexity but substantially lower $\mathrm{Var}(\log q_t)$ than
\textsc{Random}. Their similarity has a path-level explanation:
the non-EOS reveal orders of \textsc{Max-prob} and \textsc{Top-margin} closely match L2R
(\S\ref{sec:reveal_order}), so the structured strategies trace near-identical
content paths, while short-answer prompts additionally induce a tail-EOS
cascade under confidence-first rules. On downstream benchmarks (HellaSwag,
CMath, MMLU-Pro, IFEval), lower-variance confidence traces are consistently
associated with correct generations
(\S\ref{sec:downstream}).

\paragraph{Contributions.}
\begin{itemize}
\item \textbf{We document that the model is not a coherent joint distribution.} On LLaDA-2.1-mini, $\log P/n$ of a fixed target
spans $\sim 0.5$ nats/token across reveal orders, exceeding numerical
tolerance and revealing that the learned conditionals are not mutually
consistent factorizations of a single joint distribution
(\S\ref{sec:theory_valid}).
\item \textbf{We characterize why CF reveal orders are close to L2R.} Reveal-order analysis
shows that confidence-first strategies (\textsc{Max-prob} and \textsc{Top-margin})
traverse near-identical \emph{content} reveal paths to strict left-to-right decoding.
The main departure is on short-answer prompts, where confidence-first rules additionally trigger a
tail-EOS cascade that \textsc{Forced-AR} cannot exploit (\S\ref{sec:reveal_order}).
\item \textbf{We report variance as a supplementary diagnostic for performance reliability.} A uniform-spreading theorem
(Theorem~\ref{thm:uniform}) shows that, at fixed total $P$, target
recoverability is uniquely maximized at uniform $q_t$ with deviation governed by
$\mathrm{Var}(\log q_t)$, identifying $\mathrm{Var}$ as an
order-dependent diagnostic that complements $\log P$. On four downstream
benchmarks, variance is consistently associated with correctness, ruling out an
explanation based only on the poor aggregate performance of \textsc{Random}
(\S\ref{sec:downstream}).
\end{itemize}

\begin{figure}[htbp]
    \centering
    \includegraphics[width=\textwidth, trim=1 220 1 1, clip]{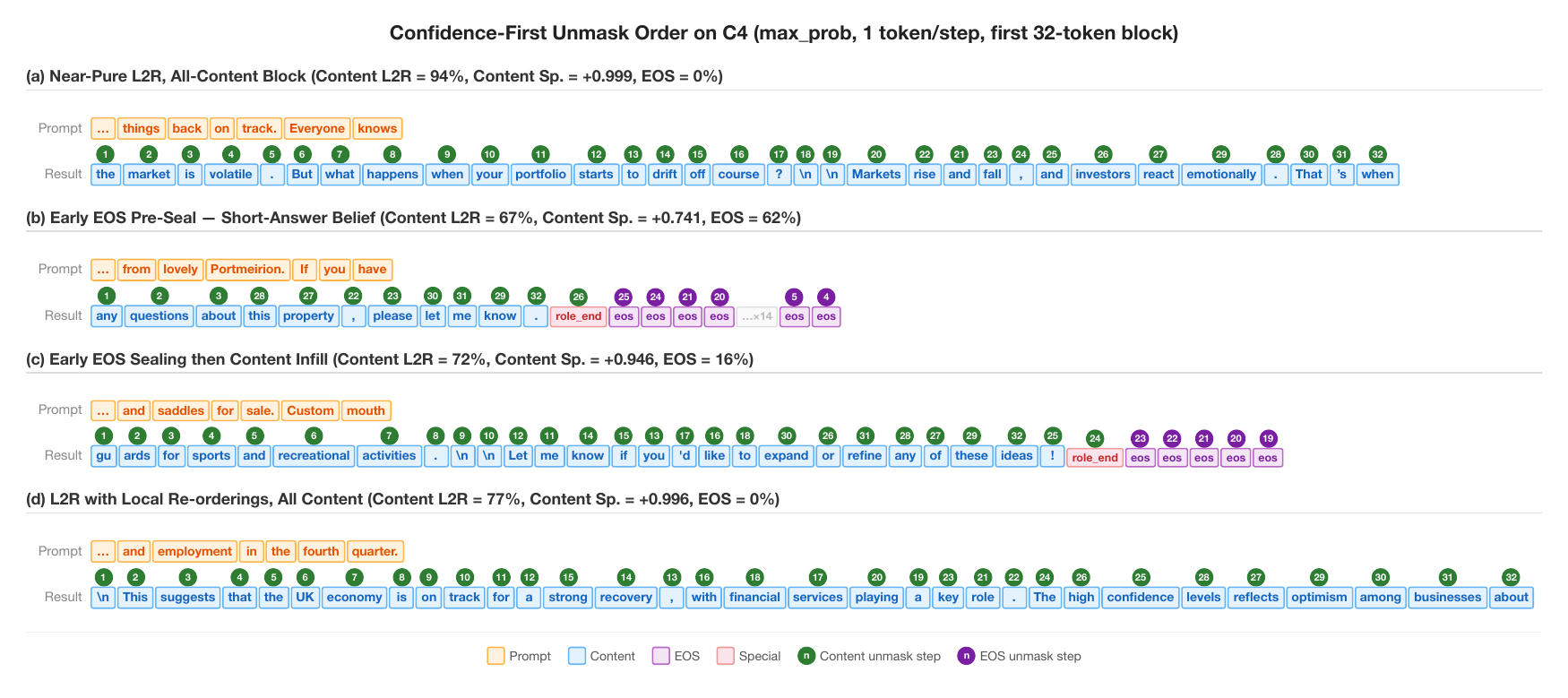}
    \caption{\textsc{Max-prob} unmask order on four representative C4 prompts (block~$0$ shown). Tokens are in reading order; circled numbers indicate the unmask step ($1$ = first). Blue = content, purple = EOS, red = special.}
    \label{fig:unmask_order}
    \vspace{-0.75em}
\end{figure}

\section{Related Work}
\label{sec:related}

\paragraph{Order-agnostic and discrete diffusion language models.}
Order-agnostic generative models date back to NADE-style
training~\cite{uria2014nade}; the family includes
XLNet~\cite{yang2019xlnet} and ARDM~\cite{hoogeboom2021ardm}, as well
as the discrete-diffusion branch: MDLM~\cite{sahoo2024mdlm},
SEDD~\cite{lou2024sedd}, and
LLaDA~\cite{nie2025llada,nie2025llada2,bie2026llada21}. These works focus on
training distributions that support flexible conditioning patterns. We
take such a trained OALM, instantiated by a dLLM in our experiments, as given
and study the inference paths it induces.

\paragraph{Decoding strategies for masked / diffusion LMs.}
A growing literature explores decoding heuristics for OALMs and dLLMs:
confidence-based unmasking
\cite{ghazvininejad2019mask, chang2022maskgit}, block-AR schedules,
RCR-style remasking and rate control~\cite{wang2025remdm}, denoising-entropy
controls~\cite{chen2025denoising_entropy}, and search over token
orders~\cite{chen2026ordertokensearch,asano2026wheretounmask}. More recent
work studies information-gain criteria for choosing the next position
\cite{yang2026infogain}. Closest in spirit are
information-theoretic analyses of generation order and parallel decoding in
masked diffusion models~\cite{zhang2026genorder}. These works
primarily ask how to choose or improve an order, or how parallel factorization
changes the sampling problem; we ask how to interpret the confidence trace
produced by an order.

\paragraph{AR vs.\ non-AR decoders.}
The longstanding question of how non-AR decoders compare to AR has been
addressed empirically in many domains, with non-AR translation as a canonical
case~\cite{gu2018nat}. Recent work has also examined fundamental limits and
trade-offs of diffusion language models~\cite{zhong2026limits} and the
ability of any-order architectures to subsume AR as a special
case~\cite{karami2026armd}. We instead operate within a single OALM,
where \textsc{Forced-AR}, CF, and random orders are decoding choices on
the same network, isolating the effect of the path $\pi$ from that of the
parameters $\theta$.

\paragraph{Order-aware training and ARMD-style hybrids.}
A complementary line of work modifies training so that the model better
matches a chosen inference order: training on the worst-case
position~\cite{kim2025icml_train_worst}, ARMD-style training that bridges
AR and any-order objectives~\cite{karami2026armd}, and DCD-style
denoising-curriculum training~\cite{shen2026dcd}. These methods change
$p_\theta$, while our experiments hold $p_\theta$ fixed and vary only the
inference-side path.
Our contribution is therefore not a new decoding rule, but an analysis of the
confidence trace induced by a fixed trained OALM, with variance serving as an
order-sensitive diagnostic complementary to likelihood.

\section{Theory of Uniform Spreading}
\label{sec:theory}

This section develops the uniform-spreading theorem, which provides a
principled criterion for comparing confidence profiles induced by
CF, \textsc{Forced-AR}, and other order-dependent decoding strategies.
In our dLLM instantiation of an OALM, the model is trained to predict masked tokens given arbitrary subsets of revealed positions, with each position independently masked at a uniformly sampled rate from $U(0,1)$ per example~\cite{nie2025llada}. Every mask pattern has nonzero training probability, so a single trained network is implicitly trained on every inference order, making the comparison of decoding orders within one model meaningful. dLLMs further use block-causal attention: at any position, the model attends bidirectionally to all tokens within the same block (including mask tokens) and causally to all preceding blocks, and in the semi-autoregressive block format blocks are processed strictly left-to-right so that reveal-order strategies choose only the order of positions inside the current block. Consequently, \textsc{Forced-AR} decoding on a dLLM is not equivalent to a true causal AR model, which strictly masks all future positions. All theory and empirical claims about $\pi$ should therefore be read as claims about within-block reveal orders, not about arbitrary global permutations of the $n$ target positions. The main text further restricts to the one-token-per-step setting, in which $\pi$ reveals exactly one position per forward pass.

\subsection{Decoding Decomposition}
\label{sec:decomposition}

Let $\mathcal{V}$ be the vocabulary and $c$ the prompt context. A decoding policy $\pi$ selects at each step~$t$ a position $\pi(t)$ to unmask given the revealed set $R_{t-1}$. Fix a target sequence $\mathbf{x}^* = (x^*_1, \ldots, x^*_n) \in \mathcal{V}^n$ conditioned on context~$c$. If the conditionals $\{p_\theta(\cdot\mid R_{t-1}, c)\}$ arise from a single coherent joint distribution~$p$ over $\mathbf{x}$ given $c$, with $P \triangleq p(\mathbf{x}^*\mid c)$, then the chain rule of~$p$ guarantees that for any policy~$\pi$,
\begin{equation}\label{eq:chain}
\prod_{t=1}^{n} q_t(\pi) = P, \qquad q_t(\pi) \triangleq p\!\left(x^*_{\pi(t)} \mid R_{t-1}, c\right),
\end{equation}
so the product is order-invariant while the distribution of $\{q_t\}$ across steps is not. For the trained $p_\theta$, however, the conditionals need not arise from any single joint, so $\prod_t q_t(\pi)$ may itself depend on $\pi$.

To isolate the allocation question from the coherence question, we analytically hold $P$ fixed and ask how the per-step profile $\{q_t(\pi)\}$ affects the \emph{target recoverability} $A(\pi)$---the probability that the decoder recovers the entire target~$\mathbf{x}^*$. For greedy decoding at step $t$, the decoder commits to $\arg\max_v p_\theta(v\mid R_{t-1},c)$, so the target is recovered only when it attains this argmax; this requires $q_t$ to exceed every competitor, and the outcome therefore depends on how the residual mass $1-q_t$ is distributed among competitor tokens. Let $f(q_t)$ denote the probability that greedy recovers the target at step~$t$; under a standard modeling assumption on the competitor distribution (deferred to App.~\ref{app:noise_model}), this probability depends on the conditional distribution only through $q_t$. Under independence of greedy success across steps (App.~\ref{app:noise_model}), $A(\pi)$ factorizes as
\begin{equation}\label{eq:greedy_acc}
A(\pi) = \prod_{t=1}^{n} f\!\bigl(q_t(\pi)\bigr)
  = \underbrace{\prod_{t=1}^{n} q_t(\pi)}_{= P\;\text{(order-invariant)}}
    \;\cdot\;
    \underbrace{\prod_{t=1}^{n} \frac{f(q_t)}{q_t}}_{D(\pi)\;\text{(order-dependent)}},
\end{equation}
so all order-dependence in $A(\pi)$ is carried by the distortion $D(\pi)$.

Since $P$ is fixed, maximizing $A(\pi)$ reduces to maximizing $D(\pi)$ under the constraint $\prod_t q_t = P$, i.e.\ maximizing $\sum_t \log f(q_t)$ with $\sum_t \log q_t$ held fixed. The key analytic property that makes this tractable is that $\log f(e^y)$ is strictly concave in $y=\log q$; see App.~\ref{app:proof_concavity} for the proof.

\subsection{Uniform Spreading}
\label{sec:uniform}

\begin{theorem}[Uniform Spreading]\label{thm:uniform}
If the log-recoverability is strictly concave with respect to the log-probability, then among all step distributions $(q_1, \ldots, q_n)$ satisfying $\prod_t q_t = P$, the target recoverability $A = \prod_t f(q_t)$ is uniquely maximized when confidence is spread uniformly across all steps:
\[
q_1 = q_2 = \cdots = q_n = P^{1/n}.
\]
\end{theorem}

\begin{proof}[Proof sketch]
Let $y_t = \log q_t$ and $L = \log P$. The constraint is $\sum_t y_t = L$, and our goal is to maximize $\log A = \sum_t \log f(e^{y_t})$. By the concavity condition above, Jensen's inequality gives
\[
\frac{1}{n}\sum_{t=1}^n \log f(e^{y_t}) \;\leq\; \log f\!\left(\exp\left(\frac{1}{n}\sum_{t=1}^n y_t\right)\right) = \log f\!\left(P^{1/n}\right),
\]
with equality if and only if $y_1 = \cdots = y_n$, yielding the upper bound $A \leq f(P^{1/n})^n$ on every per-step profile. A complete proof, including the uniqueness of the maximizer, is given in App.~\ref{app:proof_uniform}.
\end{proof}

The theorem fixes the sequence likelihood $P$ and isolates the effect of the per-step profile $\{q_t\}$ on decoding recoverability. The intuition relies on the decoding bottleneck: because $f$ is nonlinear, a single low-confidence step damages the target recoverability $A$ super-linearly. Thus, for a given total log-probability $\log P$, the optimal strategy penalizes variance and uniformly distributes the task difficulty.
The theorem deliberately begins with the simplified case where the total target likelihood $P$ is fixed: if reveal-order imbalance can affect recoverability even there, then variance becomes even more important in the trained model, where the effective likelihood may itself vary with order.

This theoretical optimum motivates the two scalar diagnostics we use to evaluate strategies in \S\ref{sec:experiments}: 
the mean log-confidence $\overline{\log q} := \tfrac{1}{n}\sum_t \log q_t$, which measures the average self-rated difficulty along the path, and the log-probability variance $\mathrm{Var}_\pi := \tfrac{1}{n}\sum_t (\log q_t - \overline{\log q})^2$, which measures deviation from the uniform optimum. Theorem~\ref{thm:uniform} predicts that, when paths have comparable $\overline{\log q}$, lower $\mathrm{Var}_\pi$ should imply higher recoverability.

\section{Experiments}
\label{sec:experiments}

\subsection{Setup and Strategies}
\label{sec:common_setup}

All experiments use LLaDA-2.1-mini~\cite{bie2026llada21} with block size $32$ under the one-token-per-step setting of \S\ref{sec:theory}, so each block takes $32$ reveal steps.
The strategy set separates practical decoders from diagnostic controls. The practical decoders are:
\begin{itemize}
  \item \textbf{\textsc{Forced-AR}}: reveal positions left-to-right within each block.
  \item \textbf{\textsc{Max-prob}} (CF): for each unrevealed position $j$ in the current block, compute its score $s_j = \max_{v\in\mathcal{V}} p_\theta(v\mid R_{t-1})$, i.e.\ the largest probability assigned to any vocabulary token at position~$j$; reveal the position with the largest $s_j$.
  \item \textbf{\textsc{Top-margin}} (CF): for each unrevealed position $j$, score it by the gap between the top two tokens of the current conditional at $j$, $s_j = p_\theta(v_1\mid R_{t-1}) - p_\theta(v_2\mid R_{t-1})$, where $v_1,v_2\in\mathcal{V}$ are those top two tokens; reveal the position with the largest $s_j$.
\end{itemize}
Two deployable controls isolate order effects:
\begin{itemize}
  \item \textbf{\textsc{Random}}: reveal a uniformly random unrevealed position at each step; the random seed is fixed across runs.
  \item \textbf{\textsc{Reverse-AR}}: reveal positions right-to-left within each block, testing whether the left-to-right direction itself explains the observed behavior.
\end{itemize}
For the fixed-sequence experiment only, we additionally use three target-aware oracle orderings. Each scores the unrevealed position $j$ by a target-dependent criterion and reveals the position with the largest score:
\begin{itemize}
  \item \textbf{\textsc{Oracle max-$q$}}: score position $j$ by the model's probability of the target token there, $s_j = q_j := p_\theta(x_j^*\mid R_{t-1})$; reveal the most confident target position first.
  \item \textbf{\textsc{Oracle margin}}: score position $j$ by the gap between the target probability and its largest competitor, $s_j = p_\theta(x_j^*\mid R_{t-1}) - \max_{v\in\mathcal{V},\,v\neq x_j^*} p_\theta(v\mid R_{t-1})$; reveal the position where the target most decisively beats its strongest competitor.
  \item \textbf{\textsc{Oracle min-$q$}}: score position $j$ by $s_j = -q_j$, i.e.\ reveal the least confident target position first.
\end{itemize}
\textsc{Oracle max-$q$} and \textsc{Oracle margin} are the target-aware counterparts of the practical CF rules \textsc{Max-prob} and \textsc{Top-margin} (replacing argmax-based scores with target-based ones), while \textsc{Oracle min-$q$} is the opposite ordering to \textsc{Oracle max-$q$} and acts as a lower-bound control. These oracle orderings are diagnostic rather than deployable because they require the known target $\mathbf{x}^*$; ``Oracle'' refers to knowledge of $\mathbf{x}^*$, not access to $p$, and probabilities are still computed by $p_\theta$.

\paragraph{Datasets and protocol.} The main C4 experiments use a fixed evaluation subset of $1{,}000$ examples selected from the C4 validation split~\cite{raffel2020c4}. Each example is tokenized into a $32$-token prefix followed by a $128$-token region, corresponding to four blocks and $128$ reveal steps under the one-token-per-step setting. The fixed-sequence validation (\S\ref{sec:theory_valid}) treats this region as the held-out target and evaluates order-dependent likelihood traces. Open-ended generation (\S\ref{sec:gen}) and the unmask-order analysis (\S\ref{sec:reveal_order}) reuse the same prefixes but let the model generate its own $128$-token continuation. Downstream evaluation (\S\ref{sec:downstream}) uses the same decoding implementation on HellaSwag~\cite{zellers2019hellaswag}, CMath~\cite{wei2023cmath}, MMLU-Pro~\cite{wang2024mmlupro}, and IFEval~\cite{zhou2023ifeval}, with task-specific prompt lengths and stopping rules detailed in App.~\ref{app:downstream_protocol}.

\paragraph{Metrics.} The two theory metrics from \S\ref{sec:theory} are $\overline{\log q}$ and $\mathrm{Var}(\log q)$, where $q_t$ is the probability assigned at the revealed position to the target token (fixed-sequence scoring) or to the selected token (generation); in generation we also report their content-only variants, which exclude EOS tokens. Output-quality metrics are EOS\% (fraction of generated EOS tokens), External PPL (GPT-2 Large~\cite{radford2019gpt2}), and Distinct-3 (fraction of unique trigrams~\cite{li2016diversity}). The fixed-sequence validation additionally reports single-step argmax accuracy and a Gini-style concentration coefficient over per-step within-block self-information as auxiliary statistics.

\subsection{Uniform-Spreading Validation: Oracle Orderings on Fixed Sequences}
\label{sec:theory_valid}

We first test the two diagnostics motivated by Theorem~\ref{thm:uniform} in a controlled setting where the $128$-token continuation of each C4 example is held fixed as the target $\mathbf{x}^*$ rather than generated (Fig.~\ref{fig:theory_valid}). Knowing $\mathbf{x}^*$ lets us evaluate the target-aware oracle orderings of \S\ref{sec:common_setup} alongside the three non-CF baselines (\textsc{Forced-AR}, \textsc{Reverse-AR}, \textsc{Random}); all statistics are computed per sequence and averaged over the $1{,}000$ examples.

\paragraph{Finding 1: the learned conditionals are not order-coherent.}
A coherent joint $p^*$ would, by the chain rule, give the same total log-likelihood $\log P = \sum_t \log p^*(x^*_{\pi(t)}\mid R_{t-1})$ for every reveal order $\pi$. In practice $\log P/n$ ranges from $-3.00$ (Forced-AR) to $-3.49$ (Oracle min-$q$), a $\sim 0.5$-nat-per-token spread (Fig.~\ref{fig:theory_valid}a). Thus $p_\theta(\cdot\mid R_{t-1})$ depends materially on the path, and $\log P$ alone is not order-invariant at the resolution needed for strategy comparison.

\paragraph{Finding 2: $\mathrm{Var}_\pi$ is a non-redundant axis driven by bottleneck steps.}
Across the six strategies in Fig.~\ref{fig:theory_valid}a, $\mathrm{Var}_\pi$ is correlated with $\log P/n$ (Spearman $\rho \approx -0.77$) but not collinear with it, and there are non-monotone pairs such as \textsc{Forced-AR} vs.\ \textsc{Oracle max-$q$} (nearly identical $\log P/n$, Var differs by $\sim\!2\!\times$) and \textsc{Random} vs.\ \textsc{Reverse-AR} (Var ranked oppositely to $\log P/n$). The mechanism behind this extra axis is visible in Fig.~\ref{fig:theory_valid}b: high-variance orderings push the within-block Lorenz curve farther below the uniform diagonal, i.e.\ self-information is concentrated in a few bottleneck steps rather than spread evenly. We therefore report both: $\log P$ captures the path's absolute model confidence, while $\mathrm{Var}_\pi$ captures the severity of bottleneck steps along the path.

\subsection{Structured Strategies Traverse Near-Identical Reveal Orders}
\label{sec:reveal_order}
\label{sec:gen}

We now drop the ground-truth assumption: each strategy decodes its own continuation from a prompt, so the per-step $q_t$ is the model's own argmax probability rather than the probability assigned to a known $x^*$.
This removes the oracle strategies from the testable set but adds two things the fixed-sequence regime cannot offer:
(i) an open-ended generation distribution, which matches practical deployment;
(ii) downstream-quality measurements (external perplexity, Distinct-3, EOS\%) that are only defined when the model produces text.

\begin{table}
\centering
\caption{Generated-sequence diagnostics and open-ended quality on $1{,}000$ C4 prompts. Each strategy generates its own continuation with one-token-per-step decoding and no correction; content-only metrics exclude EOS tokens. Best per metric is in \textbf{bold}.}
\label{tab:main}
\small
\setlength{\tabcolsep}{6pt}
\begin{tabular}{lcccc}
\toprule
Metric & Forced-AR & Max-prob & Top-margin & Random \\
\midrule
Mean $\log q$ $\uparrow$ & $-$0.631 & $-$0.585 & \textbf{$-$0.566} & $-$1.330 \\
Content mean $\log q$ $\uparrow$ & $-$0.644 & \textbf{$-$0.634} & $-$0.649 & $-$1.482 \\
Var$(\log q)$ $\downarrow$ & 0.392 & \textbf{0.286} & 0.315 & 0.990 \\
Content Var$(\log q)$ $\downarrow$ & 0.399 & \textbf{0.299} & 0.344 & 1.038 \\
\midrule
EOS\%$\downarrow$ & \textbf{2.4\%} & 7.9\% & 11.9\% & 12.4\% \\
External PPL $\downarrow$ & 7.56 & 7.53 & \textbf{7.48} & 11.65 \\
Distinct-3 $\uparrow$ & \textbf{0.928} & 0.907 & 0.884 & 0.896 \\
\bottomrule
\end{tabular}
\vspace{-0.7em}
\end{table}

\paragraph{Finding 1: on content tokens, structured strategies reduce to near-L2R reveal paths.}
On block~$0$, after excluding EOS reveals, the Spearman rank correlation with strict L2R is $89.1\%$ for \textsc{Max-prob} and $80.7\%$ for \textsc{Top-margin}. The two adaptive confidence-first rules therefore traverse \emph{content} paths close to L2R, so the three structured orderings see nearly the same sequence of contexts $\{R_{t-1}\}$ at reveal time and consequently cluster together on $\overline{\log q}$, $\mathrm{Var}(\log q)$, and external quality in Table~\ref{tab:main}; \textsc{Random}, whose path has no L2R structure, is the outlier on every metric. Panels~(a) and~(d) of Fig.~\ref{fig:unmask_order} illustrate this regime: panel~(a) is a near-pure L2R trace with every step advancing the rightward frontier, and panel~(d) is an all-content trace where \textsc{Max-prob} permutes only a few adjacent positions (local swaps among competing high-confidence neighbors) without disturbing the global rank order---consistent with the high but sub-unity content Spearman. The underlying driver is locality: the position immediately right of the prompt boundary has the strongest mutual information with the revealed context ($d{=}1$ neighbor), so confidence-first decoding tends to begin there and expand rightward.

\paragraph{Finding 2: confidence-first strategies pre-seal the EOS tail once the answer's right boundary becomes certain.}
Once EOS tokens are included, the picture diverges: the EOS fraction in Table~\ref{tab:main} is $2.4\%$ for \textsc{Forced-AR} but $7.9\%$ and $11.9\%$ for \textsc{Max-prob} and \textsc{Top-margin}. The mechanism is visible in Fig.~\ref{fig:unmask_order}(b,c): the model becomes confident about where the generation should end before it has finished writing the answer, so EOS has non-trivial probability at the rightmost tail positions while interior content is still masked. \textsc{Forced-AR}, bound to strict L2R, never exploits this; confidence-first rules open the tail. In panel~(b) only three content tokens are revealed L2R before, at step~$4$, the rightmost position commits to EOS at $0.50$ confidence, after which its neighbors activate as high-confidence EOS ($0.85$--$0.97$) and seventeen tail positions seal right-to-left in one cascade; only at step~$22$ does decoding return to the middle. Panel~(c) is a milder variant, with tail EOS committed before the answer is fully written but interleaved with content infill. Both are driven by the same trigger---growing certainty about the right boundary outweighing remaining content confidence---but differ only in when and how sharply it fires. The content Spearman of Finding~1 excludes EOS precisely to factor out this effect, which is why the content metric stays close to $1$ across all four panels despite visibly different EOS timing.

\begin{table}
\centering
\caption{Downstream accuracy (\%) per (benchmark, strategy) cell under greedy one-token-per-step decoding; mean content $\mathrm{Var}(\log q)$ in parentheses. \textbf{min-V} / \textbf{max-V} are prompt-level selectors (argmin-Var / argmax-Var): per prompt, pick the strategy with the lowest / highest content $\mathrm{Var}$ and read off its correctness. $\Delta =$ min-V $-$ max-V (pp). Best acc per row in \textbf{bold}; lowest Var \underline{underlined}.}
\label{tab:downstream}
\small
\setlength{\tabcolsep}{3pt}
\begin{tabular}{lrrrrrrr}
\toprule
Dataset & \multicolumn{4}{c}{Acc $\uparrow$ \,(Var $\downarrow$)} & min-V $\uparrow$ & max-V $\uparrow$ & $\Delta$ \\
\cmidrule(lr){2-5}
 & Forced & Max & Margin & Rand. & & & \\
\midrule
HellaSwag & \textbf{79.74}~(0.138) & 78.80~(\underline{0.090}) & 78.38~(0.101) & 75.76~(0.571) & 78.72 & 75.79 & $+2.9$pp \\
CMath     & \textbf{91.71}~(0.067) & 90.26~(\underline{0.042}) & 90.89~(0.045) & 69.22~(0.354) & 90.79 & 69.10 & $+21.7$pp \\
MMLU-Pro  & \textbf{54.78}~(0.110) & 53.12~(\underline{0.081}) & 52.03~(0.100) & 19.62~(0.573) & 52.83 & 19.99 & $+32.8$pp \\
IFEval    & 74.12~(0.289) & \textbf{75.05}~(\underline{0.216}) & 71.90~(0.255) & 67.28~(0.797) & 74.49 & 68.39 & $+6.1$pp \\
\bottomrule
\end{tabular}
\vspace{-0.7em}
\end{table}

\subsection{Downstream Tasks: Reasoning, Math, and Instruction Following}
\label{sec:downstream}

Theorem~\ref{thm:uniform} does not directly prove task accuracy; rather, it motivates uneven confidence profiles as a diagnostic of bottlenecked decoding trajectories. We next test whether this diagnostic transfers to externally graded correctness on standard benchmarks, where each generation receives a binary label rather than a self-reported likelihood. We run the four pure one-token-per-step orderings on four benchmarks and log the per-token log-probability of every revealed token along the way.

\begin{table}
\centering
\caption{Per-cell Var-vs-correctness on the $94{,}852$ graded generations behind Table~\ref{tab:downstream}. One row per (dataset, strategy) cell. $n_T/n_F$ are counts of correct/incorrect samples; $\overline{\mathrm{Var}}(T)$ and $\overline{\mathrm{Var}}(F)$ are the mean content $\mathrm{Var}(\log q)$ over correct/incorrect samples; gap $= \overline{\mathrm{Var}}(F) - \overline{\mathrm{Var}}(T)$ (positive $=$ lower Var when correct); MW $p$ is the one-sided Mann--Whitney $U$ test for $\mathrm{Var}(T) < \mathrm{Var}(F)$; $\rho$ is the Spearman correlation between per-sample $\mathrm{Var}(\log q)$ and the binary correctness flag in that cell.}
\label{tab:within_cell_main}
\small
\setlength{\tabcolsep}{3.5pt}
\begin{tabular}{llrrrrrrr}
\toprule
Dataset & Strategy & $n_T$ & $n_F$ & $\overline{\mathrm{Var}}(T)$ & $\overline{\mathrm{Var}}(F)$ & gap & MW $p$ & $\rho$ \\
\midrule
HellaSwag & Forced-AR  & 8007 & 2035 & 0.134 & 0.152 & $+$0.018 & $<\!10^{-4}$ & $-0.143$ \\
HellaSwag & Max-prob   & 7913 & 2129 & 0.086 & 0.103 & $+$0.017 & $<\!10^{-4}$ & $-0.168$ \\
HellaSwag & Top-margin & 7871 & 2171 & 0.096 & 0.117 & $+$0.021 & $<\!10^{-4}$ & $-0.180$ \\
HellaSwag & Random     & 7608 & 2434 & 0.544 & 0.654 & $+$0.109 & $<\!10^{-4}$ & $-0.177$ \\
\midrule
CMath     & Forced-AR  & 1007 &   91 & 0.066 & 0.082 & $+$0.016 & $1.5\!\times\!10^{-3}$ & $-0.089$ \\
CMath     & Max-prob   &  991 &  107 & 0.039 & 0.068 & $+$0.029 & $<\!10^{-4}$ & $-0.129$ \\
CMath     & Top-margin &  998 &  100 & 0.043 & 0.065 & $+$0.022 & $<\!10^{-4}$ & $-0.114$ \\
CMath     & Random     &  760 &  338 & 0.291 & 0.495 & $+$0.204 & $<\!10^{-4}$ & $-0.394$ \\
\midrule
MMLU-Pro  & Forced-AR  & 6591 & 5441 & 0.078 & 0.148 & $+$0.070 & $<\!10^{-4}$ & $-0.303$ \\
MMLU-Pro  & Max-prob   & 6391 & 5641 & 0.052 & 0.113 & $+$0.061 & $<\!10^{-4}$ & $-0.337$ \\
MMLU-Pro  & Top-margin & 6260 & 5772 & 0.058 & 0.144 & $+$0.086 & $<\!10^{-4}$ & $-0.342$ \\
MMLU-Pro  & Random     & 2361 & 9671 & 0.439 & 0.606 & $+$0.167 & $<\!10^{-4}$ & $-0.177$ \\
\midrule
IFEval    & Forced-AR  &  401 &  140 & 0.280 & 0.315 & $+$0.036 & $6.8\!\times\!10^{-2}$ & $-0.064$ \\
IFEval    & Max-prob   &  406 &  135 & 0.206 & 0.247 & $+$0.041 & $2.9\!\times\!10^{-2}$ & $-0.081$ \\
IFEval    & Top-margin &  389 &  152 & 0.240 & 0.293 & $+$0.054 & $1.5\!\times\!10^{-3}$ & $-0.127$ \\
IFEval    & Random     &  364 &  177 & 0.780 & 0.830 & $+$0.049 & $5.9\!\times\!10^{-2}$ & $-0.067$ \\
\bottomrule
\end{tabular}
\vspace{-0.7em}
\end{table}

\begin{table}
\centering
\caption{Two-axis downstream accuracy ($\%$) after median-splitting samples within each benchmark by content $\overline{\log q}$ and content $\mathrm{Var}(\log q)$, pooling the four strategies. The diagnostic is the joint location of a generation in the $(\overline{\log q}, \mathrm{Var})$ plane. Higher is better for $\overline{\log q}$; lower is better for Var. Best per row is in \textbf{bold}.}
\label{tab:two_axis_quadrant}
\small
\begin{tabular}{lrrrr}
\toprule
Dataset & hi-$\overline{\log q}$/lo-Var & hi-$\overline{\log q}$/hi-Var & lo-$\overline{\log q}$/lo-Var & lo-$\overline{\log q}$/hi-Var \\
\midrule
HellaSwag & \textbf{84.7} & 83.7 & 77.3 & 73.5 \\
CMath     & 92.7 & 89.6 & \textbf{94.1} & 77.0 \\
MMLU-Pro  & \textbf{62.3} & 60.7 & 34.4 & 22.7 \\
IFEval    & \textbf{77.9} & 72.1 & 64.7 & 67.3 \\
\bottomrule
\end{tabular}
\vspace{-0.7em}
\end{table}

\paragraph{Setup.}
We use the same LLaDA-2.1-mini checkpoint and one-token-per-step decoding implementation as above, with EOS early stopping and no correction procedure. The reveal order is the only factor varied across cells. For every revealed token we log $(\text{position},\,\text{token id},\,\log q_t)$, where $q_t = p_\theta(\hat{x}_t \mid R_{t-1})$ is the model's confidence in its own selected token at reveal time. EOS tokens are excluded from log-probability statistics; further decoding details are in App.~\ref{app:downstream_protocol}.
HellaSwag~\cite{zellers2019hellaswag} ($N{=}10{,}042$, validation set; commonsense 4-way MCQ); CMath~\cite{wei2023cmath} ($N{=}1{,}098$, test split; Chinese primary-school arithmetic word problems); MMLU-Pro~\cite{wang2024mmlupro} ($N{=}12{,}032$, test split; 10-way reasoning-heavy MCQ); IFEval~\cite{zhou2023ifeval} ($N{=}541$; rule-based instruction following with both prompt-strict and instruction-strict accuracies). All four strategies see the same prompts.

\paragraph{Finding 1: structured paths form a low-variance, higher-accuracy tier.}
Table~\ref{tab:downstream} shows the same tiering as the C4 experiments: \textsc{Random} has the lowest accuracy and the highest content $\mathrm{Var}(\log q)$ on every benchmark, while the three structured strategies cluster within $\pm 1.6$~pp with no stable winner inside the low-Var tier. This matches Theorem~\ref{thm:uniform}: the prediction is the separation between high-Var and low-Var paths, not an ordering among near-overlapping structured ones. Using content $\mathrm{Var}$ as a prompt-level selector (argmin-Var) matches the best single fixed strategy to within ${\sim}2$~pp on all four benchmarks and consistently exceeds the argmax-Var selector by $\Delta = +2.9$ to $+32.8$~pp: the low-variance path positively predicts correctness relative to the high-variance one.

\paragraph{Finding 2: the $\mathrm{Var}$--correctness link is not driven by strategy-level effects.}
To rule out that the pooled association is driven by \textsc{Random} alone, we restrict to the $4\times 4 = 16$ (benchmark, strategy) cells (Table~\ref{tab:within_cell_main}), in each of which both benchmark and strategy are held fixed. Correct samples have strictly lower mean $\mathrm{Var}$ than incorrect ones in \emph{all $16$} cells, with the one-sided Mann--Whitney $U$ test highly significant in most cells; the per-cell Spearman $\rho(\mathrm{Var},\,\mathbf{1}_{\text{correct}})$ is negative in every cell, including the structured-only ones. Thus $\mathrm{Var}$ genuinely tracks correctness rather than reflecting a strategy-level outlier.

\paragraph{Finding 3: mean confidence and variance are complementary axes.}
We next ask whether the Var signal is already captured by the classical $\overline{\log q}$. Splitting samples within each benchmark by median content $\overline{\log q}$ and median content $\mathrm{Var}(\log q)$ yields four quadrants (Table~\ref{tab:two_axis_quadrant}). At matched $\overline{\log q}$, low-$\mathrm{Var}$ is never worse than high-$\mathrm{Var}$. CMath is especially striking: the lo-$\overline{\log q}$/lo-Var quadrant reaches $94.1\%$, above the hi-$\overline{\log q}$/hi-Var quadrant at $89.6\%$---a low-confidence but low-variance generation is more reliable than a high-confidence but high-variance one, indicating that in some regimes $\mathrm{Var}$ carries signal beyond $\overline{\log q}$. We therefore recommend reporting $\overline{\log q}$ and $\mathrm{Var}(\log q)$ jointly rather than collapsing them into a single scalar.

\section{Discussion and Conclusion}
\label{sec:conclusion}

On the trained OALM studied here, $\overline{\log q}$ is not order-invariant: the ${\sim}0.5$-nat/token chain-rule deviation in LLaDA-2.1-mini is far above numerical tolerance, so mean confidence alone mixes content difficulty with the ordering of easy and hard positions. On natural language, the structured strategies (\textsc{Forced-AR}, \textsc{Max-prob}, \textsc{Top-margin}) trace near-identical \emph{content} reveal paths while \textsc{Random} forms the high-variance, lower-quality tier; on short-answer prompts, confidence-first rules additionally trigger a tail-EOS cascade that \textsc{Forced-AR} cannot exploit (\S\ref{sec:reveal_order}). Theorem~\ref{thm:uniform} supplies the order-dependent complement $\mathrm{Var}(\log q_t)$: even at fixed confidence budget, bottlenecked traces are less recoverable than uniformly spread ones---a minimal fixed-$P$ case, since the trained model is more complex with $\prod_t q_t(\pi)$ also order-dependent. Variance is thus not a substitute for accuracy or likelihood but a supplementary diagnostic for whether a path is stable or driven by bottlenecked, order-sensitive steps. The downstream results are consistent with this diagnostic view: low-variance paths form the higher-accuracy tier, and the Var--correctness association survives after fixing both benchmark and strategy (\S\ref{sec:downstream}). We therefore read $(\overline{\log q}, \mathrm{Var}(\log q))$ as a joint diagnostic plane: low $\overline{\log q}$ + high Var = poor and unstable; high $\overline{\log q}$ + high Var = average confidence driven by a few favorable steps; similar $\overline{\log q}$ but lower Var = the more stable order or policy. We recommend reporting them jointly. The primary strategy-level distinction we observe is between structured low-variance paths and \textsc{Random}; restoring chain-rule coherence is a training-side problem outside our scope.

\section{Limitations}
\label{sec:limitations}
(i)~While we use dLLMs as a representative OALM testbed, all experiments rely on one checkpoint and scale (LLaDA-2.1-mini); transfer to larger LLaDA variants and other architectures (MDLM, SEDD, XLNet) remains open. The proposed joint reporting of $\overline{\log q}$ and $\mathrm{Var}(\log q)$ is model-agnostic, but the empirical rankings should be interpreted at this scale.
(ii)~The main experiments focus on one-token-per-step decoding. Parallel decoding, where multiple tokens are committed simultaneously, introduces a separate batch-coupling effect; extending the variance diagnostic to jointly account for order and batch size is left for future work.
(iii)~Our variance analysis is diagnostic rather than a new deployable decoding algorithm. The prompt-level min-V selector uses logged trajectories to test whether low variance predicts correctness; designing an online policy that reduces variance without oracle comparisons is a separate problem.

\clearpage
\bibliographystyle{plain}
\bibliography{refs}

\clearpage
\appendix
\renewcommand{\thefigure}{S\arabic{figure}}
\setcounter{figure}{0}

\renewcommand{\thetable}{S\arabic{table}}
\setcounter{table}{0}

\section{Proofs and Derivations}\label{app:proofs}

\subsection{Setup: Chain Rule and Distortion Decomposition}\label{app:noise_model}

\paragraph{Notation.}
The prompt or context is denoted by $c$, the target sequence by $\mathbf{x}^*=(x_1^*,\ldots,x_n^*)$, and the reveal order by $\pi$, where $\pi(t)$ is the position revealed at step~$t$. The set of already revealed tokens before step~$t$ is $R_{t-1}$. We write $q_t$ for the model probability assigned to the target token $x^*_{\pi(t)}$ at that step, and $P=p(\mathbf{x}^*\mid c)$ for the reference sequence probability. The vocabulary is $\mathcal{V}$ with size $S=|\mathcal{V}|$; for LLaDA-2.1-mini, $S=32{,}000$. All logarithms are natural logarithms.

\paragraph{Chain-rule baseline.}
If the conditionals came from a coherent joint distribution $p(\mathbf{x}\mid c)$, then for any reveal order~$\pi$,
\[
\prod_{t=1}^{n} q_t(\pi)=P,\qquad
q_t(\pi)=p\!\left(x^*_{\pi(t)}\mid R_{t-1},c\right),\qquad
P=p(\mathbf{x}^*\mid c).
\]
The product is independent of the order. Empirically, however, LLaDA-2.1 changes $\log P/n$ by up to $0.49$ nats/token across reveal orders, so the learned conditionals should not be treated as exact factorizations of one joint distribution. We therefore use the chain-rule identity as an analytic reference point and separately track order-sensitive diagnostics.

\paragraph{Greedy decoding and distortion.}
Greedy decoding succeeds at step~$t$ only if the target token wins the local argmax. Let $f(q_t)$ denote the probability that this step recovers the target token when the model assigns target probability $q_t$. The effective sequence recovery probability is
\[
A(\pi)=\prod_{t=1}^{n} f(q_t).
\]
Multiplying and dividing by the chain-rule product isolates the order-dependent part:
\begin{equation}\label{eq:distortion_decomp}
A(\pi)=\prod_{t=1}^{n} f(q_t)
=\underbrace{\prod_{t=1}^{n}q_t}_{=\,P\ \text{(order-invariant reference)}}\;
\underbrace{\prod_{t=1}^{n}\frac{f(q_t)}{q_t}}_{D(\pi)\ \text{(order-dependent)}} .
\end{equation}
Thus $A=P\cdot D$. With $P$ fixed, maximizing $A$ is equivalent to maximizing the distortion term $D(\pi)$.

\paragraph{Independence clarification.}
The product form $A(\pi)=\prod_t f(q_t)$ is a conditional step-success model: at step~$t$, the success probability is evaluated conditioned on all previously revealed tokens being correct. The random logit noises $\epsilon_v^{(t)}$ are taken independent across steps in this model, while the confidences $q_t$ themselves need not be independent because they are determined by the revealed context $R_{t-1}$. Thus the factorization concerns the noise variables used to model local greedy decisions, not independence of the confidence trace.

\subsection{Explicit Greedy-Success Function}\label{app:greedy_success}

We next derive the one-step success probability $f(q)$ under the single-competitor Gaussian-logit model used in the proof. Let $q\in(0,1)$ be the model probability assigned to the target token. Greedy decoding assigns each token a noisy logit
\[
\ell_v=\log p_v+\epsilon_v,\qquad
\epsilon_v\overset{\mathrm{iid}}{\sim}\mathcal{N}(0,\sigma^2),
\]
and selects the largest logit. The target must beat every other token; equivalently, it must beat the strongest competitor, or runner-up. We therefore model this binding competitor directly: the target has probability $q$, and the runner-up carries the residual mass $1-q$. Concentrating the residual mass on one competitor is the hardest single-competitor allocation for the target to beat, so the resulting $f$ is a conservative one-step success model. This gives the two logits
\[
\ell_{\mathrm{target}}=\log q+\epsilon_{\mathrm{target}},\qquad
\ell_{\mathrm{comp}}=\log(1-q)+\epsilon_{\mathrm{comp}}.
\]
The target is selected iff
\[
\ell_{\mathrm{target}}>\ell_{\mathrm{comp}}
\quad\Longleftrightarrow\quad
\Delta\triangleq\epsilon_{\mathrm{comp}}-\epsilon_{\mathrm{target}}
<\log q-\log(1-q).
\]
Since $\Delta\sim\mathcal{N}(0,2\sigma^2)$, standardizing by its standard deviation $\sigma\sqrt2$ yields
\begin{equation}\label{eq:fclosed}
f(q)=\Pr\!\left[\Delta<\log q-\log(1-q)\right]
=\Phi\!\left(\frac{\log q-\log(1-q)}{\sigma\sqrt2}\right).
\end{equation}
Now set $y=\log q$ and define the standardized margin
\begin{equation}\label{eq:zdef}
z(y)=\frac{\log q-\log(1-q)}{\sigma\sqrt2}
=\frac{y-\log(1-e^y)}{\sigma\sqrt2}.
\end{equation}
Then
\[
\boxed{\;f(q)=\Phi(z(y))\;}
\]
where $\Phi$ is the standard-normal CDF,
\[
\Phi(a)=\Pr[Z\le a]=\int_{-\infty}^{a}\frac{1}{\sqrt{2\pi}}e^{-u^2/2}\,du,
\qquad Z\sim\mathcal{N}(0,1).
\]
The argument $z(y)$ is a standardized margin: its numerator is the target-vs-competitor log-probability gap, and its denominator $\sigma\sqrt2$ is the standard deviation of the noise difference. Thus $f(q)=\Phi(z(y))$ is the probability that the target beats the runner-up. This single-competitor expression is exact for the two-token comparison and does not depend on the vocabulary size. For greedy decoding, this is the per-step success probability used in the distortion decomposition, so $h=f$.

\subsection{Log-Space Objective}\label{app:log_space}

Jensen's inequality applies after converting the multiplicative constraint into a sum. Let
\[
y_t=\log q_t,\qquad L=\log P.
\]
Then $\prod_t q_t=P$ becomes
\[
\sum_{t=1}^{n} y_t=L.
\]
Define the log-recoverability function
\begin{equation}\label{eq:gdef}
g(y)\triangleq \log f(e^y)=(\log\Phi)(z(y)).
\end{equation}
The sequence log-recovery probability is therefore
\[
\log A=\sum_t \log f(q_t)=\sum_t g(y_t).
\]
The relevant concavity is concavity of $g$ as a function of $y=\log q$, not concavity of $f$ as a function of $q$.

\subsection{Concavity of the Log-Recoverability Function}\label{app:proof_concavity}

\begin{proof}
We prove concavity for $g(y)=(\log\Phi)(z(y))$ by explicitly differentiating both the outer function $\log\Phi$ and the inner function $z(y)$. In the outer-derivative calculation below, $z$ denotes a generic real argument of $\log\Phi$; after those derivatives are computed, we substitute back $z=z(y)$.

\paragraph{Step 1: chain rule.}
By the chain rule,
\begin{equation}\label{eq:gpp}
g''(y)=
\underbrace{(\log\Phi)''\!\bigl(z(y)\bigr)[z'(y)]^2}_{\text{Term 1}}
+
\underbrace{(\log\Phi)'\!\bigl(z(y)\bigr)z''(y)}_{\text{Term 2}} .
\end{equation}
We now compute each factor in this expression.

\paragraph{Step 2: outer derivatives.}
Let
\[
r(z)\triangleq\frac{\varphi(z)}{\Phi(z)}>0
\]
be the inverse Mills ratio, where $\varphi$ is the standard-normal density. Then
\[
(\log\Phi)'(z)=\frac{\Phi'(z)}{\Phi(z)}
=\frac{\varphi(z)}{\Phi(z)}
=r(z)>0.
\]
Using $\varphi'(z)=-z\varphi(z)$,
\[
(\log\Phi)''(z)=r'(z)
=\frac{\varphi'(z)\Phi(z)-\varphi(z)^2}{\Phi(z)^2}
=-\,r(z)\bigl(z+r(z)\bigr).
\]
Moreover, $z+r(z)>0$ for all $z\in\mathbb{R}$: for $z\ge0$ this follows from $r(z)>0$, and for $z<0$ the Mills bound $\Phi(z)<\varphi(z)/|z|$ gives $r(z)>|z|=-z$. Thus $(\log\Phi)''(z)<0$~\cite{bagnoli2005logconcave}.

\paragraph{Step 3: inner derivatives.}
From
\[
z(y)=\frac{y-\log(1-e^y)}{\sigma\sqrt2},
\]
and $\frac{d}{dy}\log(1-e^y)=-e^y/(1-e^y)$, we get
\[
z'(y)=\frac{1}{\sigma\sqrt2(1-e^y)}>0,\qquad
z''(y)=\frac{1}{\sigma\sqrt2}\frac{e^y}{(1-e^y)^2}>0.
\]

\paragraph{Step 4: substitute and factor.}
Substituting $(\log\Phi)'=r$ and $(\log\Phi)''=-r(z+r)$ into~\eqref{eq:gpp}, with $z=z(y)$, gives
\[
g''(y)
=-r(z)\bigl(z+r(z)\bigr)[z'(y)]^2+r(z)z''(y).
\]
Both terms contain the common positive factor $r(z)$, so
\begin{equation}\label{eq:gpp_factor}
\boxed{\;
g''(y)=r(z)\Bigl[z''(y)-\bigl(z+r(z)\bigr)[z'(y)]^2\Bigr].
\;}
\end{equation}

\paragraph{Step 5: reduce to a scalar criterion.}
Since $r(z)>0$, the sign of $g''$ is the sign of the bracket in~\eqref{eq:gpp_factor}. The condition $g''(y)<0$ is equivalent to
\[
\bigl(z+r(z)\bigr)[z'(y)]^2>z''(y).
\]
Substituting the expressions for $z'$ and $z''$, and writing $q=e^y$, gives
\[
\frac{z+r(z)}{2\sigma^2(1-q)^2}>
\frac{q}{\sigma\sqrt2(1-q)^2}
\quad\Longleftrightarrow\quad
\frac{z+r(z)}{2\sigma^2}>
\frac{q}{\sigma\sqrt2}
\quad\Longleftrightarrow\quad
z+r(z)>\sqrt2\,\sigma q.
\]
Thus
\begin{equation}\label{eq:concavity_criterion}
\boxed{\;
g''(y)<0
\quad\Longleftrightarrow\quad
z(q)+r\!\bigl(z(q)\bigr)>\sqrt2\,\sigma q,
\qquad
z(q)=\frac{\log q-\log(1-q)}{\sigma\sqrt2}.
\;}
\end{equation}
This criterion is the useful point. The left-hand side $z+r(z)$ is always positive by the Mills-ratio argument above, while the right-hand side $\sqrt2\,\sigma q$ scales linearly with the noise level. The sign is therefore controlled by $\sigma$, not by a separate assumption that $q\ll1$: as $q\to0$ the right-hand side also vanishes, and moderate or large values of $q$ can still satisfy the criterion when $\sigma$ is small.

\paragraph{Step 6: small-$\sigma$ sign.}
The noise scale $\sigma$ is the smoothing parameter that turns the deterministic argmax event into the smooth probability $\Phi(z)$; in the zero-noise limit this recovers the hard threshold $\mathbf{1}[q>\tfrac12]$. It is therefore natural to analyze the $\sigma\to0$ regime. Let $s=\sigma\sqrt2$ and $z=\mathrm{logit}(q)/s$. We show that for sufficiently small $\sigma$, the criterion~\eqref{eq:concavity_criterion} holds for all $q\in(0,1)$.

If $q\ge\tfrac12$, then $z\ge0$ and $z+r(z)\ge r(0)=\sqrt{2/\pi}$. Also $\sqrt2\,\sigma q=sq<s$, so the criterion holds whenever $\sigma\le1/\sqrt\pi$.

If $q<\tfrac12$, then $z\to-\infty$ as $\sigma\to0$. Mills' expansion gives
\[
z+r(z)=-\frac{1}{z}+O(z^{-3})
=\frac{s}{|\mathrm{logit}(q)|}+O(s^3)
=\frac{\sigma\sqrt2}{|\mathrm{logit}(q)|}+O(\sigma^3).
\]
Therefore
\[
z+r(z)-\sqrt2\,\sigma q
=\frac{\sigma\sqrt2}{|\mathrm{logit}(q)|}
\Bigl(1-q|\mathrm{logit}(q)|\Bigr)+O(\sigma^3).
\]
Since
\[
\max_{q\in(0,1/2)} q|\mathrm{logit}(q)|\approx0.2785<1,
\]
with the maximum attained near $q\approx0.218$, the factor $1-q|\mathrm{logit}(q)|$ is bounded below by approximately $0.7215>0$. Hence the leading term is positive, and the criterion holds for sufficiently small~$\sigma$ in this case as well.

Combining both cases,
\[
\boxed{\;
\exists\,\sigma_0>0:\ \forall\,\sigma\le\sigma_0,\ \forall\,q\in(0,1),
\quad g''(y)<0.
\;}
\]
Numerically, the admissible range extends to approximately $\sigma_0\approx1.05$; at $\sigma=0.1$, the swept margin $z+r-\sqrt2\sigma q$ remains positive over $q\in(0,1)$.
\end{proof}

\paragraph{Extension to multiple competitors.}
For $k$ competitors, let $\epsilon_0$ be the target noise and $\epsilon_1,\ldots,\epsilon_k$ the competitor noises. Define
\[
W_k\triangleq\max_{i=1,\ldots,k}(\epsilon_i-\epsilon_0),
\qquad
F_{W_k}(w)=\Pr[W_k\le w].
\]
If the residual mass $1-q$ is spread uniformly over $k$ competitors, the one-step success probability is
\[
f_k(q)=F_{W_k}\!\left(\log\frac{kq}{1-q}\right),
\]
where $(\epsilon_1-\epsilon_0,\ldots,\epsilon_k-\epsilon_0)$ is a mean-zero multivariate Gaussian with covariance $\sigma^2(I+\mathbf{1}\mathbf{1}^{\top})$. For $k=1$, $W_1\sim\mathcal{N}(0,2\sigma^2)$ and this reduces to~\eqref{eq:fclosed}. For $k\ge2$, $F_{W_k}$ has no elementary closed form, but $w\mapsto\log F_{W_k}(w)$ is concave by Pr\'ekopa's theorem~\cite{prekopa1973logarithmic}. Applying the same two-term decomposition as above gives the same mechanism. The two extremes are $k=1$ (a single dominant competitor) and $k=S-1$ (residual mass spread over all non-target vocabulary items); in realistic softmax distributions, only a small effective set of competitors usually carries substantial mass, so the regime lies between these extremes. Numerical checks with Zipfian competitor masses preserve the sign in the regime used here. Thus the competitor distribution mainly changes the effective noise scale; the single-competitor model is a clean base case rather than a load-bearing vocabulary-size assumption.

\subsection{Proof of Theorem~\ref{thm:uniform} (Uniform Spreading)}\label{app:proof_uniform}

\begin{proof}
We maximize
\[
\log A=\sum_{t=1}^{n}g(y_t)
\qquad\text{subject to}\qquad
\sum_{t=1}^{n}y_t=L.
\]
Since $g$ is concave by App.~\ref{app:proof_concavity}, Jensen's inequality gives
\[
\frac1n\sum_{t=1}^{n}g(y_t)
\le
g\!\left(\frac1n\sum_{t=1}^{n}y_t\right)
=g\!\left(\frac{L}{n}\right).
\]
Equality holds iff $y_1=\cdots=y_n=L/n$, equivalently
\[
\boxed{\;q_1=q_2=\cdots=q_n=P^{1/n}\;}.
\]
At this point $A$ takes the value $[f(P^{1/n})]^n$.

For uniqueness, equivalently form the Lagrangian
\[
J=\sum_t g(y_t)-\lambda\left(\sum_t y_t-L\right).
\]
The stationarity condition is $g'(y_t)=\lambda$ for all~$t$. Strict concavity makes $g'$ strictly decreasing and therefore one-to-one, so all $y_t$ must be equal. Concavity then makes this stationary point the global maximizer.
\end{proof}

\paragraph{Intuition.}
Greedy decoding is a chain of local decisions: if one step has too little confidence and loses the local argmax, the whole target sequence fails even if other steps are easy. Under a fixed total log-probability budget, spreading difficulty evenly across steps avoids such bottlenecks. This is the same water-filling principle in another form: fixed budget plus concave per-step utility implies a uniform optimum.

\paragraph{Why variance is the diagnostic.}
The theorem identifies the uniform point but not the local cost of moving away from it. Let
\[
\bar y=\frac{L}{n}=\frac1n\sum_t y_t.
\]
Expanding each $g(y_t)$ around $\bar y$ gives
\[
g(y_t)\approx
g(\bar y)+g'(\bar y)(y_t-\bar y)
+\frac12 g''(\bar y)(y_t-\bar y)^2.
\]
Summing over $t$, the linear term vanishes because $\sum_t(y_t-\bar y)=0$, and
\[
\sum_t(y_t-\bar y)^2=n\,\mathrm{Var}(\log q).
\]
Therefore
\[
\boxed{\;
\log A
\approx
\underbrace{n\,g(\bar y)}_{\text{set by mean }\overline{\log q}}
-\frac{n}{2}\bigl|g''(\bar y)\bigr|\,\mathrm{Var}(\log q).
\;}
\]
Thus, when the mean log-confidence is matched, the leading loss relative to the uniform optimum is proportional to $\mathrm{Var}(\log q)$. This gives the two diagnostics used in the experiments:
\[
\overline{\log q}=\frac1n\sum_t\log q_t,
\qquad
\mathrm{Var}_\pi=\frac1n\sum_t\left(\log q_t-\overline{\log q}\right)^2.
\]
The prediction is that, among paths with comparable $\overline{\log q}$, lower $\mathrm{Var}_\pi$ should correspond to higher recoverability. This is why the experiments report both $\overline{\log q}$ and $\mathrm{Var}(\log q)$ rather than collapsing the confidence trace to one scalar.

\section{Downstream-Evaluation Protocol and Diagnostics}
\label{app:downstream_protocol}

This appendix details the evaluation protocol for the downstream experiments
reported in \S\ref{sec:downstream} and Table~\ref{tab:downstream}, then gives
the paired diagnostic tests behind the per-sample uniformity claims.

\paragraph{Benchmarks and grading.}
\begin{itemize}
    \item \textbf{HellaSwag}~\cite{zellers2019hellaswag}: validation split,
        $10{,}042$ items total. Each example presents a context plus four
        candidate endings; greedy CoT generation~\cite{wei2022chain}, then ending selected by
        normalized log-probability of the four candidates given the
        generation.
    \item \textbf{CMath}~\cite{wei2023cmath}: full test split,
        $n{=}1{,}098$. Chinese primary-school arithmetic word problems; the
        last numeric span in the generated response is extracted and
        compared to the reference answer (string-equality after numeric
        normalization).
    \item \textbf{MMLU-Pro}~\cite{wang2024mmlupro}: test split,
        $12{,}032$ items. $0$-shot CoT~\cite{kojima2022large} with
        OpenCompass-aligned~\cite{opencompass2023} ``Answer: X'' extraction (regex on the last
        line); $10$-way MCQ.
    \item \textbf{IFEval}~\cite{zhou2023ifeval}: train split, full
        $n{=}541$ prompts ($834$ atomic instructions). Rule-based scorer;
        we report both prompt-strict (sample is correct iff every listed
        instruction is satisfied) and instruction-strict accuracy
        (instruction-level mean).
\end{itemize}

\paragraph{Strategies and decoding budget.}
We evaluate the four pure one-token-per-step orderings of \S\ref{sec:gen}:
\textsc{Forced-AR}, \textsc{Max-prob}, \textsc{Top-margin}, and
\textsc{Random}, all decoded greedily (softmax temperature $T{=}0$) with a maximum generation length of $16{,}384$ tokens, end-of-sequence (EOS) early-stop, and no token-correction.
Same prompt formatting across all four strategies. Generation uses a block-causal KV-cached one-token-per-step driver, and all four benchmarks are run on the same LLaDA-2.1-mini checkpoint.

\paragraph{Per-token logging.}
For every revealed token we record $(\text{position},\,\text{token id},\,\log q_t)$, where $q_t$ is the model's argmax probability at the time the position is committed (computed from the same forward pass that the strategy uses to choose the position). The content statistics in Table~\ref{tab:downstream_detail} ($\overline{\log q}$, $\mathrm{Var}(\log q)$, $n_{\text{tok}}$) are computed from this per-token log after dropping EOS tokens.

\paragraph{Sample-size and CI.}
For binary accuracy with observed proportion $\hat p$, the $95\%$ CI half-width
$1.96\sqrt{\hat p(1-\hat p)/n}$ is at most $\pm 4.2$pp at $n{=}541$ (IFEval),
$\pm 3.0$pp at $n{=}1{,}098$ (CMath), $\pm 1.0$pp at $n{=}10{,}042$
(HellaSwag), and $\pm 0.9$pp at $n{=}12{,}032$ (MMLU-Pro).
The Random-vs-rest gaps in Table~\ref{tab:downstream} ($6$--$35$pp) all exceed
this floor by an order of magnitude on every benchmark; the within-uniformity
gaps ($\leq 1.6$pp on HellaSwag and IFEval) do not. The per-sample tests in Table~\ref{tab:within_cell_main} are not subject to this $\sqrt{n}$ floor because they are paired tests with at least $500$ samples per cell.

\subsection{Downstream Diagnostic Tests}

\paragraph{Per-cell content statistics.}
Table~\ref{tab:downstream_detail} reports the per-cell content
$\overline{\log q}$, $\mathrm{Var}(\log q)$, and $n_{\mathrm{tok}}$
underlying the downstream analysis in the main text. All quantities exclude
EOS tokens and are averaged over the samples in the corresponding cell.

\begin{table}[ht]
\centering\small
\caption{Content-token statistics for each (dataset, strategy) cell behind
Table~\ref{tab:downstream}. Content excludes EOS tokens. Best per row in the
indicated direction is in \textbf{bold}; $n_{\mathrm{tok}}$ is reported for
completeness and has no monotone direction in the theory.}
\label{tab:downstream_detail}
\begin{tabular}{lcccc}
\toprule
Dataset & Forced-AR & Max-prob & Top-margin & Random \\
\midrule
\multicolumn{5}{l}{\textit{Content mean $\log q \uparrow$}} \\
HellaSwag & $-0.255$ & \textbf{$-0.211$} & \textbf{$-0.211$} & $-0.548$ \\
CMath     & $-0.132$ & \textbf{$-0.112$} & $-0.115$ & $-0.538$ \\
MMLU-Pro  & $-0.194$ & \textbf{$-0.170$} & $-0.183$ & $-1.408$ \\
IFEval    & $-0.470$ & \textbf{$-0.462$} & $-0.491$ & $-1.180$ \\
\midrule
\multicolumn{5}{l}{\textit{Content $\mathrm{Var}(\log q) \downarrow$}} \\
HellaSwag & 0.138 & \textbf{0.090} & 0.101 & 0.571 \\
CMath     & 0.067 & \textbf{0.042} & 0.045 & 0.354 \\
MMLU-Pro  & 0.110 & \textbf{0.081} & 0.100 & 0.573 \\
IFEval    & 0.289 & \textbf{0.216} & 0.255 & 0.797 \\
\midrule
\multicolumn{5}{l}{\textit{Content length $n_{\mathrm{tok}}$}} \\
HellaSwag &  138 &  135 &  137 &  130 \\
CMath     &  410 &  513 &  433 &  408 \\
MMLU-Pro  & 2488 & 2815 & 2922 & 3997 \\
IFEval    &  455 &  528 &  533 & 2309 \\
\bottomrule
\end{tabular}
\end{table}

\paragraph{Per-cell Spearman correlations.}
Table~\ref{tab:rho_var_corr} reports Spearman $\rho$ between per-sample
$\mathrm{Var}(\log q)$ and the binary correctness flag; the correlation is
negative in all sixteen cells (four datasets $\times$ four strategies),
supporting a monotonic rather than tail-driven relationship. The correlations
are strongest on MMLU-Pro for the structured strategies ($-0.30$ to $-0.34$)
and on CMath for Random ($-0.39$), and weakest on IFEval
($|\rho| \leq 0.13$). The pattern is consistent with the prediction that
Var-correctness coupling tightens on benchmarks with a single-token
bottleneck (MCQ ${>}$ math word problems ${>}$ free-form instruction
following).

\begin{table}[ht]
\centering\small
\caption{Per-cell Spearman $\rho$ of per-sample $\mathrm{Var}(\log q)$ with binary correctness.}
\label{tab:rho_var_corr}
\begin{tabular}{lrrrr}
\toprule
Dataset & Forced-AR & Max-prob & Top-margin & Random \\
\midrule
HellaSwag & $-0.143$ & $-0.168$ & $-0.180$ & $-0.177$ \\
CMath     & $-0.089$ & $-0.129$ & $-0.114$ & $-0.394$ \\
MMLU-Pro  & $-0.303$ & $-0.337$ & $-0.342$ & $-0.177$ \\
IFEval    & $-0.064$ & $-0.081$ & $-0.127$ & $-0.067$ \\
\bottomrule
\end{tabular}
\end{table}

\section{Greedy vs.\ Sampling: Additional Robustness Checks}
\label{app:greedy_vs_sample}

The theory and the main-text experiments are stated for greedy decoding ($T{=}0$). Here we ask two questions. \textbf{(Q1)} Does sampling help? That is, for a fixed reveal-order strategy, does softmax sampling at $T{>}0$ improve over greedy? \textbf{(Q2)} How does the separation of \textsc{Max-prob} decoding evolve under sampling? We sweep the softmax temperature $T \in \{0, 0.3, 0.5, 0.7, 1.0, 1.3, 1.5, 2.0\}$ for the four strategies \textsc{Forced-AR}, \textsc{Max-prob}, \textsc{Top-margin}, and \textsc{Random} on the open-ended generation setup of \S\ref{sec:gen} (one-token-per-step, $1{,}000$ C4 prompts, generation length $128$), and report per-strategy quality including externally-graded GPT-2 Large perplexity (Ext-PPL). We report the median Ext-PPL throughout because the per-prompt PPL distribution is heavy-tailed (a small number of generations have PPL $>10^3$ while the bulk stays below $30$).

\begin{figure}[ht]
    \centering
    \includegraphics[width=\textwidth]{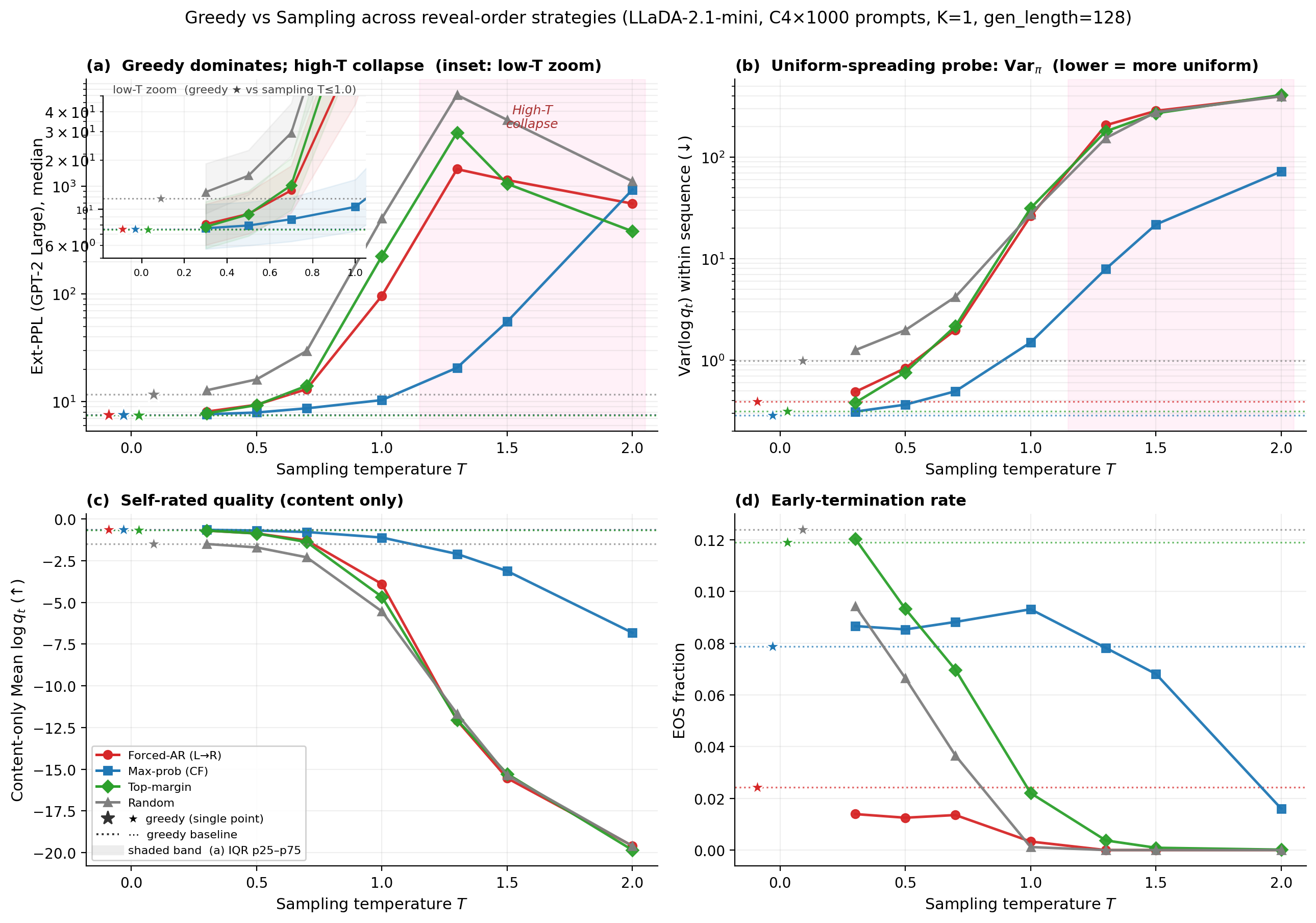}
    \caption{Greedy ($T{=}0$, shown as stars with a small horizontal offset to separate near-equal values) vs.\ sampling ($T{>}0$) across four reveal-order strategies on LLaDA-2.1-mini, $1{,}000$ C4 prompts, one-token-per-step, generation length $128$. (a) Ext-PPL (GPT-2 Large, median) over the full temperature range; inset zooms the low-$T$ regime so the four greedy baselines are visible. (b) Within-sequence log-prob variance $\mathrm{Var}(\log q_t)$ on log scale. (c) Self-rated mean content log-prob. (d) EOS fraction.}
    \label{fig:greedy_vs_sampling}
\end{figure}

\paragraph{(Q1) Greedy gives the lowest external perplexity in this sweep.}
Fig.~\ref{fig:greedy_vs_sampling}(a) answers (Q1) in the negative for this setting. For every reveal-order strategy, the greedy star sits at or below every point on the same-coloured sampling curve: external-perplexity medians are $7.56$ / $7.54$ / $7.47$ / $11.64$ for \textsc{Forced-AR}/\textsc{Max-prob}/\textsc{Top-margin}/\textsc{Random} at $T{=}0$, while the best sampling temperature for each strategy only matches (\textsc{Max-prob}: $T{=}0.3$ gives $7.64$) or underperforms greedy. This is consistent with Theorem~\ref{thm:uniform}: for a fixed policy $\pi$, injecting multiplicative per-step noise increases $\mathrm{Var}(\log q_t)$ and hence the deviation from the uniform-spreading optimum. Panel~(b) shows the same pattern: $\mathrm{Var}(\log q_t)$ increases monotonically with $T$ for all four strategies, with greedy giving the smallest variance.

\paragraph{(Q2) Sampling separates \textsc{Max-prob}.}
At $T{=}0$, the external-perplexity gap between \textsc{Forced-AR} ($7.56$), \textsc{Max-prob} ($7.54$), and \textsc{Top-margin} ($7.47$) is small; only \textsc{Random} ($11.64$) is separated from the structured strategies. At $T{>}0$, \textsc{Max-prob}'s external perplexity increases more slowly ($7.64 \to 7.94 \to 8.68 \to 10.35$ at $T \in \{0.3, 0.5, 0.7, 1.0\}$), while \textsc{Forced-AR} and \textsc{Top-margin} increase more rapidly ($13.1$ / $14.1$ at $T{=}0.7$; $96$ / $224$ at $T{=}1$). \textsc{Random} increases earliest. The same ordering appears on the uniform-spreading probe $\mathrm{Var}(\log q_t)$ in panel~(b): at $T{=}1$, \textsc{Max-prob}'s within-sequence variance is $1.50$, vs.\ $26.3$ / $31.4$ / $27.5$ for the other three, i.e.\ \textsc{Max-prob} is $\sim 20\times$ more uniform under sampling noise.
The cross-strategy spread on external perplexity (median max minus median min across the four strategies) grows from $4.2$ at $T{=}0$, to $5.1$ at $T{=}0.3$, to $20.9$ at $T{=}0.7$, to $495$ at $T{=}1$---a ${\sim}120\times$ increase as we move from greedy to moderate-temperature sampling. This supports the prediction of Theorem~\ref{thm:uniform}: the noise-tolerance margin associated with uniform spreading is small when there is little sampling noise, but grows with the noise level. Thus, although greedy remains best in this sweep, \textsc{Max-prob} changes more slowly than the other strategies under sampling.

\paragraph{High-temperature regime.}
At $T{\geq}1.3$ all four strategies enter a high-temperature regime: external perplexity exceeds $10^3$ (panel~(a), shaded band), content log-prob falls below $-12$ nats per token (panel~(c)), and the EOS fraction drops to nearly zero because the model no longer reaches a natural stopping point within the budget (panel~(d)). The shaded region in panels~(a)/(b) marks this regime; we exclude it from qualitative comparisons.

\paragraph{Relation to the main-text downstream protocol.}
The main-text downstream evaluation (\S\ref{sec:downstream}) uses greedy decoding. In this sweep, greedy is the best point on the temperature curve for each reveal-order strategy, while sampling primarily enlarges the cross-strategy gap and separates \textsc{Max-prob} from the other strategies.

\paragraph{Caveat.}
The temperatures here are applied uniformly across all positions; learnable position-dependent or step-dependent schedules (e.g.\ decreasing $T$ as decoding progresses) might shift the curves but should not affect the qualitative ordering predicted by the theory. External perplexity is scored by GPT-2 Large on the concatenated $(\text{prompt}, \text{gen-text})$, so it reflects open-ended coherence to an AR evaluator rather than task accuracy.

\end{document}